\documentclass{article}

\usepackage[preprint]{neurips_2023}

\usepackage[utf8]{inputenc} 
\usepackage[T1]{fontenc}    
\usepackage{hyperref}       
\hypersetup{
    colorlinks,
    linkcolor={red!80!black},
    citecolor={blue!80!black},
    urlcolor={blue!80!black}
}
\usepackage{url}            
\usepackage{booktabs}       
\usepackage{amsfonts}       
\usepackage{nicefrac}       
\usepackage{microtype}      
\usepackage{xcolor}         

\usepackage{amsmath}
\usepackage{amssymb}
\usepackage{amsthm}
\usepackage{thmtools, thm-restate}
\usepackage{listings}

\usepackage{hyperref}
\hypersetup{
    colorlinks,
    linkcolor={red!80!black},
    citecolor={blue!80!black},
    urlcolor={blue!80!black}
}
\usepackage{enumitem}
\usepackage{pdfcomment}
\usepackage{tikz}
\usepackage{tikz-cd}
\usetikzlibrary{arrows}
\usepackage{xcolor}
\usepackage[font=small,labelfont=bf]{caption}
\usepackage{float}

\newtheorem{theorem}{Theorem}[section]
\newtheorem{lemma}{Lemma}[theorem]
\newtheorem{proposition}[theorem]{Proposition}
\newtheorem{fact}[theorem]{Fact}
\newtheorem{corollary}[theorem]{Corollary}
\theoremstyle{definition}
\newtheorem{definition}[theorem]{Definition}
\newtheorem{example}[theorem]{Example}

\theoremstyle{remark}
\newtheorem{remark}[theorem]{Remark}

\newcommand{\sgn}{\text{sgn}}
\DeclareMathOperator{\cl}{cl}
\DeclareMathOperator{\ran}{ran}
\newcommand{\R}{\mathbb{R}}
\newcommand{\NN}{\mathbb{N}}

\newcommand{\Clus}{\mathbb{K}}

\newcommand{\bd}[1]{\mathbf{#1}}  
\DeclareMathOperator*{\argmax}{argmax}

\title{Multidimensional Hopfield Networks for clustering}

\author{%
Gergely Stomfai \\ University of Cambridge, Synerise
\And
Łukasz Sienkiewicz \\ Synerise
\And
Barbara Rychalska \\ Synerise
}

\begin{document}

\maketitle

\begin{abstract}
    We present the Multidimensional Hopfield Network (DHN), a natural generalisation of the Hopfield Network. In our theoretical investigations we focus on DHNs with a certain activation function and provide energy functions for them. We conclude that these DHNs are convergent in finite time, and are equivalent to greedy methods that aim to find graph clusterings of locally minimal cuts. We also show that the general framework of DHNs encapsulates several previously known algorithms used for generating graph embeddings and clusterings. Namely, the Cleora graph embedding algorithm, the Louvain method, and the Newman's method can be cast as DHNs with appropriate activation function and update rule. Motivated by these findings we provide a generalisation of Newman's method to the multidimensional case.
\end{abstract}

\section{Multidimensional Hopfield Networks}\label{section:mhn}
\noindent
For the reader's convenience we introduce some notation that is used throughout the paper. If $n,d$ are positive integers, then $\mathbb{M}_{n\times d}(\R)$ denotes the set of $n\times d$ matrices with entries in $\R$. If $X \in \mathbb{M}_{n\times d}(\R)$ is a matrix, then $X^T$ denotes its transpose and provided that $n = d$, the trace of $X$ is denoted by $\mathrm{Tr}(X)$. Let $F:A\rightarrow B$ be a function between the sets $A,B$. Then its range i.e. the subset
$$\big\{F(a)\,\big|\,a\in A\big\}\subseteq B$$
is denoted by $\ran(F)$.
\noindent
Now we turn to the general framework that encapsulates all the methods to follow.
\begin{definition}\label{defintion:mhn}
    Let $n,d$ be positive integers. \textit{A $d$-dimensional Hopfield network with $n$ neurons} is a triple $(W,B,F)$ such that $W \in \mathrm{M}_{n\times n}(\R)$, $B \in \mathrm{M}_{n\times d}(\R)$ and $F:\R^{d} \rightarrow \R^{d}$ is a function, called the \textit{activation function}. Further, each neuron $i$ of the DHN has a \textit{state vector} $X_i \in \R^d$.
\end{definition}
\noindent
Let us interpret the notion above in a more intuitive way, which is more familiar to machine learning practitioners. Suppose that $(W,B,F)$ is a two-dimensional Hopfield network with three neurons. Then $W \in \mathrm{M}_{3\times 3}(\R)$ is interpreted as the connection weights matrix between the three neurons.
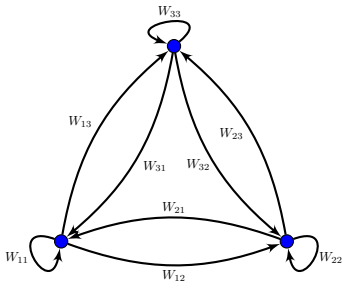
\begin{center}
\begin{tikzpicture}[scale=0.5, transform shape]
    \tikzset{vertex/.style = {shape=circle,fill=black!0!blue,draw,minimum size=1.em}}
    \tikzset{edge/.style = {->,> = latex', line width=0.8pt}}
    \node[vertex] (1) at (-3,0) {};
    \node[vertex] (2) at (3,0) {};
    \node[vertex] (3) at (0,5.196) {};

    \draw[edge] (1) to[out=340, in=200] node[below] {$W_{12}$} (2);
    \draw[edge] (2) to[out=160, in=20] node[above] {$W_{21}$} (1);

    \draw[edge] (2) to[out=100, in=320] node[below=1pt, left=2pt] {$W_{23}$} (3);
    \draw[edge] (3) to[out=280, in=140] node[below=7pt, left=-4pt] {$W_{32}$} (2);

    \draw[edge] (3) to[out=260, in=40] node[auto] {$W_{31}$} (1);
    \draw[edge] (1) to[out=80, in=220] node[auto] {$W_{13}$} (3);

    \draw[edge] (1) to[out=160, in=260, distance=1.5cm] node[below, left] {$W_{11}$} (1);
    \draw[edge] (2) to[out=20, in=280, distance=1.5cm] node[below, right] {$W_{22}$} (2);
    \draw[edge] (3) to[out=40, in=160, distance=1.5cm] node[above] {$W_{33}$} (3);
\end{tikzpicture}
\end{center}
The entries of the matrix $B\in \mathrm{M}_{3\times 2}(\R)$ are biases corresponding to each neuron and each state, while $F:\R^2\rightarrow \R^2$ is the common activation function for all the neurons in the network. The state vectors of the neurons are vectors in $\R^2$. From this understanding of the network it is clear, that DHNs are a subclass of recurrent neural networks. \\
Since recurrent neural networks determine discrete time dynamical systems on the space of neuron states, each DHN gives rise to such a dynamical system.

\begin{definition}
    Let $(W,B,F)$ be a $d$-dimensional Hopfield network with $n$-neurons. Let $X_i(t) \in \R^d$ denote the state of the $i$-th neuron at time $t\in \NN$. We consider the following two types of discrete time dynamical systems determined by $(W,B,F)$ and the initial states $X_i(0)$ of the neurons. 
    \begin{itemize}
        \item In \textit{serial mode of operation} given that the neurons are in the states
        $$X_1(t),...,X_n(t)\in \R^d$$
        at time $t \in \NN$, we pick a neuron $i \in \{1,...,n\}$ and update it's state according to the rule
        $$X_{i}(t+1) = F\left(\sum_{j=1}^nW_{ij}\cdot X_j(t) + B_i\right)$$
        while the states of all other neurons stay unchanged, i.e. $X_j(t+1) = X_j(t)$ for all $j \neq i$. Note that the dynamics of the system depends on the order in which the neurons are updated.
        \item In \textit{parallel mode of operation} given that the neurons are in the states
        $$X_1(t),...,X_n(t)\in \R^d$$
        at time $t \in \NN$, all neurons are updated simultaneously by the same rule as in serial mode (see Remark \ref{remark:activation_application_to_rows}).
    \end{itemize}
\end{definition}

\begin{remark}\label{remark:notation_for_states_of_neurons}
    Let $(W,B,F)$ be a $d$-dimensional Hopfield network with $n$-neurons operating in either serial or parallel mode. In the sequel we denote by $X(t)$ the matrix which has the neuron state vectors $X_1(t),...,X_n(t)$ as rows and call it the \textit{matrix of neuron states}.
\end{remark}
\noindent
One can find both one and two dimensional DHNs that mimic the behaviour of the original Hopfield networks introduced in \cite{hopfield:1982}.\\
In some cases, DHNs are only updated according to the parallel mode. This comes with some computational ease, and we can also be a bit more vague about the nature of the activation function, as Remark \ref{remark:full_matrix_update_in_parallel_mode} explains.

\begin{remark}\label{remark:activation_application_to_rows}
    If $F: \R^{d} \to \R^{d}$, then for a $M \in \mathrm{M}_{n\times d}(\R)$ we denote by $F(M)$ the result of applying $F$ to every row of $M$. Now let $(W,B,F)$ be a $d$-dimensional Hopfield network with $n$ neurons. Then, in parallel mode of operation the update can be written in the following form
    $$X(t+1) = F \left( W X(t) + B \right)$$
    for every $t \in \NN$. 
\end{remark}

\begin{remark}\label{remark:full_matrix_update_in_parallel_mode}
    The matrix form of the parallel update noted in Remark \ref{remark:activation_application_to_rows} shows that for this mode of operation one can set $F$ to be a function on matrices, i.e. $F: \mathbb{M}_{n \times d}(\R) \to \mathbb{M}_{n \times d}(\R)$. Although strictly speaking Definition \ref{defintion:mhn} does not allow this, as long as the network is only updated in parallel mode, it poses no practical difficulty, but allows more intricate behaviour in some cases and will prove to be useful later.
\end{remark}

\begin{example}\label{example:cleora_as_DHN}
    Consider a weighted graph with $n$ nodes and edge weight matrix $W$. Then for a fixed positive integer $d$ we set $F:\R^d \rightarrow \R^d$ to be the $l_2$-normalization function, i.e.
    $$F(x) = \frac{x}{\lVert x\rVert _2}$$
    if $x \neq 0$ and $F(0) = 0$. This gives rise to a $d$-dimensional Hopfield network $\left(W,0,F\right)$ with $n$ neurons. Further, if the entries of the initial neuron states matrix $X(0) \in \mathbb{M}_{n\times d}(\R)$ are sampled from $\mathcal{U}(-1,1)$, then $\left(W,0,F\right)$ operating in parallel mode is equivalent to running the Cleora algorithm from \cite{cleora:2021}. 
\end{example}

\begin{remark}\label{remark:range_of_activation_function}
    Let $(W,B,F)$ be a $d$-dimensional Hopfield network with $n$-neurons operating in either serial or parallel mode. Then by means of Remark \ref{remark:activation_application_to_rows} and without loss of generality we may assume that $X(t) \in \ran(F)$ for every $t\in \NN$ where
    $$\ran(F) = \big\{F(M)\,\big|\,M\in \mathbb{M}_{n\times d}(\R)\big\}$$
\end{remark}

\section{DHNs with classification function}\label{section:mhns_with_classification_activation}
\noindent
In this section we investigate the convergence properties of DHNs with a particular activation function.

\begin{definition}\label{definition:classification_function}
    The function $\cl: \mathbb{R}^d \rightarrow \mathbb{R}^d$ defined as
    $${\cl(x)}_i = \delta_i \left(\argmax_j x_j \right)$$
    for each $i \in \{1,...,d\}$ is the \textit{classification function for $\R^d$}. We also denote the image of $\cl$ in $\mathbb{R}^d$ by $\mathbb{L}_d$ and call it \textit{the label space of dimension $d$}.  
\end{definition}
\noindent
For DHNs with classification function as the activation we can provide an energy function -- a function that is decreasing along the trajectories obtained by serial mode operations.

\begin{restatable}{theorem}{energythm}\label{theorem:energy_for_DHN_with_cl_activation}
Suppose that $W \in \mathbb{M}_{n\times n}(\R)$ is a symmetric matrix with nonnegative entries on the diagonal and $B \in \mathbb{M}_{n\times d}(\R)$. Then the function
$$V\left(X\right) = -\text{Tr} \left( X^T W X + 2 X^T B \right)$$
is an energy function function for the multidimensional Hopfield network $(W, B,\cl)$ operating in serial mode.
\end{restatable}

\begin{restatable}{theorem}{convergencethm}\label{theorem:convergence_of_DHN_with_cl_activation}
    Let $(W,B, \cl)$ be a $d$-dimensional Hopfield network with $n$ neurons. Suppose that $W$ is symmetric. Then the following assertions hold.
    \begin{enumerate}[label=\emph{\textbf{(\arabic*)}}, leftmargin=*]
        \item If $W$ has nonnegative entries on the diagonal and $(W,B,\cl)$ operates in a serial mode, then the corresponding dynamical system converges to a stable state for every initial state of the neurons.
        \item If $(W,B,\cl)$ is operating in parallel mode, then the corresponding dynamical system converges to a cycle of length at most $2$ for every initial state of the neurons.
    \end{enumerate}
\end{restatable}
\noindent
For proofs we refer the reader to Appendix \ref{appendix:proofs_of_convergence_theorems}. Moreover, the results of \cite{bruck:1990} concerning the relationship between Hopfield networks and greedy algorithms solving graph min-cut problem can also be generalised to DHNs with classification function. DHNs though, solve a generalised version of min-cut problem. We discuss this at length in Appendix \ref{appendix:greedy_graph_clustering}.      

\section{DHNs for optimising modularity}
\noindent
The modularity matrix of a graph was defined in \cite{newman:2006}. We exhibit that two well known methods of maximising graph modularity, namely Newman's method from \cite{newman:2006} and the Louvain method from \cite{blondel:2008} can be viewed as certain DHNs. Extended discussion on the topic, like proofs etc. can be found in Appendix \ref{appendix:details_on_louvain_and_newman}  

\subsection{The Louvain method}
\noindent
The Louvain method \cite{blondel:2008} is a popular graph clustering method, based around the heuristic idea of greedy local search, using the modularity of the clustering as an objective function to maximise. Strictly speaking, after reaching a local optimum, the Louvain method merges nodes in the same cluster together -- this part of the algorithm is irrelevant for our analysis.

\begin{proposition}\label{proposition:louvain_hopfield_net_equivalence}
    For every weighted graph $G$ there exists a DHN with classification function as activation such that running the Louvain method on $G$ with any initial choice of clusters can be cast as running this DHN in serial mode of operation and some initialization.
\end{proposition}

For the proof see Proposition \ref{proposition:louvain_hopfieldnet_equivalence} in \ref{appendix:details_on_louvain_and_newman}.

\subsection{Newman's and other iterative methods}
\noindent
Similarly one can easily construct a DHN which is equivalent to the power method used in finding the leading eigenvector of matrices. Using this observation one can also implement Newman's method from \cite{newman:2006}. For details see \ref{appendix:details_on_louvain_and_newman}.\\  
Writing Newman's method as a DHN allows us to extend it to the case of multi-label clustering, simply by considering a DHN of higher dimensions. Let $n,d$ be positive integers. Then a subset
$$V_d(\R^n) = \big\{O \in \mathbb{M}_{n\times d}(\R)\,\big|\,O^TO = I_d\big\} \subseteq \mathbb{M}_{n\times d}(\R)$$
is \textit{the Stiefel manifold of orthonormal $d$-frames in $\R^n$}. Let $P_{V_d(\R^n)} : \mathbb{M}_{n \times d}(\R) \to V_d(\R^n)$ be such that $P_{V_d(\R^n)}(M) = \argmax_{S \in V_d(\R^n)} \text{Tr} \left( S^T M \right)$ for every $M \in \mathbb{M}_{n\times d}(\R)$.
We set $F=P_{V_d(\R^n)}$ - thus ensuring that $X(t) \in V_d(\R^n)$. See Remark \ref{remark:full_matrix_update_in_parallel_mode} about using matrix valued functions for $F$. The existence of $P_{V_d(\R^n)}$ can be verified constructively, for example by considering the QR decomposition algorithm (see \cite{allaire:2008} for more details).

\begin{lstlisting}[numbers=left,caption={\label{listing:generalised_newman}\textbf{Generalised Newman method}.},belowskip=5mm]
    input modularity matrix $Q$
    initialise neuron states matrix of $(Q, \bd{0}, P_{V_d(\R^n)})$ randomly 
    until the neuron states matrix of $(Q, \bd{0}, P_{V_d(\R^n)})$ converges
        update neurons of $(Q, \bd{0}, P_{V_d(\R^n)})$ in parallel mode
    return $\cl(\text{neuron states matrix})$
\end{lstlisting}

\begin{table}
    \centering
    \begin{tabular}{c||c|c}
        & $F =P_{V_d(\R^n)}$ & $F = \cl$ \\
        \hline \hline
        serial & SGNM & LMS \\
        \hline  
        parallel & GNM & PLMS
    \end{tabular}
    \caption{\label{table:propagation_methods} Propagation methods naturally arising from generalising Newman's method to the multidimensional case. The abbreviations refer to: SGNM serial generalized Newman method, LMS: Louvain method search, which refers to the first phase of the Louvain method, PLMS: parallel Louvain method search, GNM: Generalised Newman method.}
\end{table}
\noindent
By varying the mode of operation and the post-processing function we obtain several new methods that can again be used to find clusterings. We summarise some of these in Table \ref{table:propagation_methods}. The methods GNM, SGNM and PLMGS are previously unexplored according to our knowledge. They are methods that inherit properties of both the Louvain and Newman's methods to some extent, and therefore might be of interest for finding composite methods combining the strengths of the previously existing ones. We present some experimental result obtained from running these methods on well-known graphs in Table \ref{table:experiemnt_results}. \\
We emphasize that the settings we explored are not exhaustive, and they only serve here as a demonstration. We highlight that finding the optimal settings poses an open question.

\begin{table}
    \centering
    \begin{tabular}{c||c|c|c|c}
        & Cora & Citeseer & PubMed & Photoes \\
        \hline \hline
        LMS & 0.5510 & 0.6659 & 0.5602 & \textbf{0.6898}  \\
        \hline  
        GNM & 0.5754 & 0.5901 & 0.4489 & 0.4901 \\
        \hline
        GNM + one iteration LMS & \textbf{0.7147} & \textbf{0.7292} & \textbf{0.6580} & 0.6786 \\
        \hline
        PLMS & 0.5019 & 0.5292 & 0.4245 & 0.6042
    \end{tabular}
    \small
    \caption{\label{table:experiemnt_results} Comparison of modularity value for some of the methods from Table \ref{table:propagation_methods}. All parallel methods were halted when converged. Both parallel methods were run using $64$ dimensional neuron states, and thus were limited to $64$ clusters at any point in time. We delegate benchmarking of SGNM to further experiments.}
\end{table}

\section{Conclusions}
\noindent
We introduced DHNs as generalisation of Hopfield networks. By providing an energy function and generalising results of \cite{bruck:1990} we supported the claim that DHNs with classification function as activation are natural multidimensional analogues of classical Hopfield networks. We also provide equivalence between DHNs and graph clusterings methods like the Louvain method and the Newman's method. This enables us to generalise these two approaches, which gives rise to promising and effective modularity maximisers. Note that we did not discuss associative content-addressable memory capacity of DHNs even for classification function, or methods for training DHNs. We leave the exploration of this subject as a promising direction for further research.       

\small
\bibliographystyle{apalike}
\bibliography{refs}

\appendix
\section{Appendices}

\subsection{Proofs of convergence theorems}\label{appendix:proofs_of_convergence_theorems}
\noindent
In this appendix we are supplying the reader with rigorous proofs of Theorems \ref{theorem:energy_for_DHN_with_cl_activation} and \ref{theorem:convergence_of_DHN_with_cl_activation}.

\energythm*
\noindent
For the proof we need the following technical lemma:

\begin{lemma}\label{lemma:delta_of_standard_energy_function}
    Let $W \in \mathbb{M}_{n\times n}(\R)$ be a symmetric matrix and let $B \in \mathbb{M}_{n\times d}(\R)$. Suppose that $X, \Delta \in \mathbb{M}_{n\times d}(\R)$. Define
    $$H = WX + B$$
    Let $V:\mathbb{M}_{n\times d}(\R)\rightarrow \R$ be the function given by
    $$V\left(X\right) = -\mathrm{Tr}\left(X^TWX + 2X^TB\right)$$
    Then
    $$V(X + \Delta) - V(X) = -2\cdot \mathrm{Tr} \left( \Delta^T H\right) - \mathrm{Tr} \left(\Delta^T W \Delta \right)$$
\end{lemma}
\begin{proof}[Proof of the lemma]
    We have
    $$V(X + \Delta) - V(X) = $$
    $$=-\mathrm{Tr} \bigg( (X+\Delta)^T W (X +\Delta) + 2 (X +\Delta)^T B \bigg) +  \mathrm{Tr} \bigg( X^T W X + 2 X^T B \bigg) =$$
    $$= -\mathrm{Tr} \bigg( X^T W \Delta + \Delta^T W X + \Delta^T W \Delta + 2\Delta^T B\bigg)$$
    Since $W$ is symmetric, and the trace of a square matrix equals the trace of its transpose, we can write
    $$\mathrm{Tr}\big(X^T W \Delta\big) = \mathrm{Tr}\bigg(\big(X^T W \Delta\big)^T\bigg) = \mathrm{Tr}\big(\Delta^T W^T X\big) = \mathrm{Tr}\big(\Delta^T W X\big)$$
    Since $\mathrm{Tr}$ is linear, we obtain
    $$\mathrm{Tr} \bigg( X^T W \Delta + \Delta^T W X + \Delta^T W \Delta + 2\Delta^T B\bigg) = 
    \mathrm{Tr} \bigg(2\Delta^T W X + 2\Delta^T B\bigg) + \mathrm{Tr} \left(\Delta^T W \Delta \right) = $$
    $$= 2\cdot \mathrm{Tr} \bigg(\Delta^T \big(W X + B\big)\bigg) + \mathrm{Tr} \left(\Delta^T W \Delta \right) = 2\cdot \mathrm{Tr} \left( \Delta^T H\right) + \mathrm{Tr} \left(\Delta^T W \Delta \right)$$
    This proves the lemma.
\end{proof}

\begin{proof}[Proof of the theorem]
    Assume that $(W,B,\cl)$ operates in serial mode. Suppose that $X\in \mathbb{M}_{n\times d}(\R)$ is the matrix of neuron states of $(W,B,\cl)$ at some time $t \in \NN$. Suppose that the states of the neurons of $(W,B,\cl)$ at time $t+1$ are given by the rows of $X + \Delta$, where $\Delta \in \mathbb{M}_{n\times d}(\R)$. According to Lemma \ref{lemma:delta_of_standard_energy_function} we have
    $$V\left(X + \Delta\right) - V\left(X\right) = -2\cdot \mathrm{Tr} \left( \Delta^T H \right) - \mathrm{Tr} \left(\Delta^T W \Delta \right)$$
    where $H = WX + B$. Since the network is operating in serial mode, there is a unique neuron say $k$ which is updated at timestamp $t$. Let $j_1,j_2 \in \{1,...,d\}$ be such that
    $$j_1 = \argmax_{j}H_{kj} = \argmax_{j}\bigg\{\sum_{i=1}^nW_{ki}X_{ij} + B_{kj}\bigg\},\,X_{kj_2} = 1$$
    If $j_1 = j_2$, then $\Delta = 0$. Hence we may assume that $j_1\neq j_2$. Then
    $$\Delta_{kj_1}  = 1,\,\Delta_{kj_2} = -1$$
    and these are the only nonzero entries of $\Delta$. Taking this into consideration we have
    $$-2\cdot \mathrm{Tr} \left( \Delta^T H\right) - \mathrm{Tr} \left(\Delta^T W \Delta \right) = -2\cdot \left(H_{kj_1} - H_{kj_2}\right) - W_{j_1j_1} - W_{j_2j_2}$$
    Since $W$ has nonnegative entries on the diagonal, we derive that
    $$-2\cdot \left(H(X)_{kj_1} - H(X)_{kj_2}\right) - W_{j_1j_1} - W_{j_2j_2} \leq -2\cdot \left(H_{kj_1} - H_{kj_2}\right)$$
    Using the facts that
    $$j_1 = \argmax_{j}H_{kj},\,X_{kj_2} = 1,\,j_1\neq j_2, $$
    we obtain $H_{kj_1} - H_{kj_2} > 0$. Thus, in summary, we proved that
    $$V\left(X + \Delta\right) - V\left(X\right) \leq 0$$
    and the equality holds if and only if states of neurons does not change during update at timestamp $t$. 
\end{proof}

\convergencethm*
\begin{proof}
    The assertion \textbf{(1)} is an immediate consequence of the existence of an energy function - which was proved in Theorem \ref{theorem:energy_for_DHN_with_cl_activation}.\\
    The proof of \textbf{(2)} relies on the same idea as the proof of the corresponding statement in \cite{bruck:1990}. Suppose that $(W,B,\cl)$ runs in a parallel mode with sequence of neurons states $\{X(t)\}_{t\in \NN}$. Consider a $d$-dimensional Hopfield network $(\hat{W}, \hat{B}, \cl)$ with $2n$ neurons, where
    $$\hat{W} = \begin{pmatrix} 0 & W \\ W & 0 \end{pmatrix} \quad \hat{B} = \begin{pmatrix} B \\ B \end{pmatrix}$$
    The network $(\hat{W}, \hat{B}, \cl)$ is bipartite, with partitions
    $$\{1,2,...,n\},\,\{1 + n,2 + n,...,n + n\}$$
    Further, it's connection satisfies the following:
    \begin{itemize}
        \item Neuron $i$ in the first partition is connected to neuron $j + n$ in the second partition by directed edge of weight $W_{ij}$.
        \item  Neuron $j + n$ in the second partition is connected to neuron $i$ in the first partition by directed edge of weight $W_{ji}$, which is also equal to $W_{ij}$.
        \item Neuron $i$ in the first partition has bias vector $B_i$.
        \item Neuron $j + n$ in the second partition has bias vector $B_j$.
    \end{itemize}
    We show that there is a DHN with connections and biases as above, which in a serial mode is equivalent to $(W, B, \cl)$. To do so, we first provide an initial state for the network, followed by a sequence of serial mode operations, and we prove that the state of $(W, B, \cl)$ can be deduced from the state of the extended network. First, we set the state of the network at $t = 0$. We set:
    $$\hat{X}_{i + n}(0) = X_i(0)$$
    for every $i\in \{1,...,n\}$ and we set $\hat{X}_i(0)$ to be an arbitrary vector in $\R^d$ for every $i\in \{1,...,n\}$. Then, we update the nodes cyclically according to the following sequence.
    $$1,2...,n,1 + n,2 + n,...,n + n$$ 
    Using the architecture of $(\hat{W}, \hat{B}, F)$, one can prove by mathematical induction that
    $$\hat{X}_i\left((2t + 1)\cdot n\right) = X_i(2t + 1),\,\hat{X}_{i + n}\left(2t\cdot n\right) = X_i(2t)$$
    for each $i \in \{1,...,n\}$ and $t \in \NN$. Since $\hat{W}$ is symmetric with zero diagonal, we use \textbf{(1)} to derive that $\{\hat{X}(t)\}_{t\in \NN}$ is in a stable state for all sufficiently large times $t$. Let $\hat{X} \in \mathbb{M}_{2n\times d}(\R)$ be this stable state. Then 
    $$X_i(2t) = \hat{X}_{i + n} = X_{i}(2t + 2)$$
    and
    $$X_{i}(2t + 1) = \hat{X}_i = X_{i}(2t + 3)$$
    for each $i \in \{1,...,n\}$ and for all sufficiently large times $t$. Thus $\{X(t)\}_{t\in \NN}$ converges to a cycle of length $2$.        
\end{proof}

\subsection{Greedy graph clustering methods}\label{appendix:greedy_graph_clustering}
\noindent
The work of \cite{bruck:1990} provides an excellent starting point to understand the connection between Hopfield networks and graph cuts encoded by their states -- running the Hopfield network yields cuts with smaller cut-values over iterations. In this section we present a generalisation of this result to multidimensional Hopfield networks with classification function activation. In particular, we prove that DHNs optimise graph clusterings in an analogous way to how Hopfield networks optimise graph cuts. \\

\begin{definition}
    Let $n,d$ be positive integers and let $G$ be a graph with set of nodes $\{1,...,n\}$. \textit{A $d$-clustering of $G$} is a partitioning of $\{1,...,n\}$ into $d$ disjoint subsets. The subsets are referred to as \textit{clusters}.
\end{definition}
\noindent
Next we define the measure of quality of graph clusterings. 

\begin{definition}
    Let $n,d$ be positive integers and let $G$ be a weighted graph with set of nodes $\{1,...,n\}$ nodes and edge weight matrix $W \in \mathbb{M}_{n\times n}(\R)$. Let $\{c_1,...,c_d\}$ be a $d$-clustering of $G$. Then
    $$C_G\left(c_1,...,c_d\right) = \sum_{k \neq l} \sum_{i \in c_k}\sum_{j\in c_l} W_{ij}$$
    is \textit{the $d$-cut value of $\{c_1,...,c_d\}$ with respect to $W$}.
\end{definition}
\noindent
In order to relate graph clusterings with multidimensional Hopfield networks with $\cl$ activation we introduce the following special class of matrices.

\begin{definition}
Let $n,d$ be positive integers. A matrix in $\mathbb{M}_{n\times d}(\R)$ with rows in $\mathbb{L}_d$ is a \textit{clustering matrix}. The set of all clustering matrices in $\mathbb{M}_{n\times d}(\R)$ is denoted by $\Clus_{n\times d}(\R)$. 
\end{definition}

\begin{remark}\label{remark:clustering_matricies_is_range_of_classification}
According to Remark \ref{remark:range_of_activation_function} we have $\ran(\cl) = \Clus_{n\times d}(\R)$.
\end{remark}
\noindent
As the name suggests, every clustering matrix matrix $\Clus_{n\times d}(\R)$ encodes a clustering of $\{1,...,n\}$. Indeed, if $X$ is a clustering matrix with $n$ rows and $d$ columns, then
$$c_k = \big\{i\in \{1,...,n\}\,\big|\,X_{ik} = 1\big\}$$
for $k \in \{1,...,d\}$ is a clustering
\begin{remark}\label{remark:nonuniqueness_of_clustering_matrix}
    Note that there are multiple clustering matrices in $\Clus_{n\times d}(\R)$ associated with the same clustering of $\{1,...,n\}$. In fact, if $X \in \Clus_{n\times d}(\R)$ and $Y \in \mathbb{M}_{d\times d}(\R)$ is an arbitrary permutation matrix, then the matrix $XY$ corresponds to permuting the labels of the clusters defined by $X$, thus $X$ and $XY$ encode the same clustering.
\end{remark}
\noindent
We can express the d-cut value of a clustering in terms of an associated clustering matrix.

\begin{fact}\label{fact:min_cut_quadratic_form}
    Let $G$ be a graph with set of nodes $\{1,...,n\}$ and edge weight matrix $W \in \mathbb{M}_{n\times n}(\R)$. Let $d$ be a positive integer and let $\{c_1,...,c_d\}$ be a $d$-clustering of $G$. If $X \in \Clus_{n\times d}(\R)$ is a clustering matrix encoding $\{c_1,...,c_d\}$, then
    $$C_G\left(c_1,...,c_d\right) = \mathrm{Vol}(G) - \mathrm{Tr}\left(X^TWX\right)$$
    where 
    $$\mathrm{Vol}(G) = \sum_{i,j=1}^nW_{ij}$$
\end{fact}
\begin{proof}
    Note that
    $$\mathrm{Tr}\left(X^TWX\right) = \sum_{k=1}^d\sum_{i,j\in c_k}W_{ij}$$
    and since
    $$\sum_{k=1}^d\sum_{i,j\in c_k}W_{ij} = \left(\sum_{k=1}^d\sum_{i,j\in c_k}W_{ij} + \sum_{k\neq l}\sum_{i\in c_k}\sum_{j\in c_l}W_{ij}\right)  - \sum_{k\neq l}\sum_{i\in c_k}\sum_{j\in V_l}W_{ij} = $$
    $$= \underbrace{\sum_{i=1}^n\sum_{j=1}^nW_{ij}}_{\mathrm{Vol}(G)} - \underbrace{\sum_{k\neq l}\sum_{i\in c_k}\sum_{j\in c_l}W_{ij}}_{C_G\left(c_1,...,c_d\right)}$$
    we derive that
    $$C_G\left(c_1,...,c_d\right) = \mathrm{Vol}(G) - \mathrm{Tr}\left(X^TWX\right)$$
\end{proof}

\noindent
Using this observation and results from preceeding sections, we present the following results.

\begin{corollary}\label{corollary:DHN_with_cl_and_zero_bias_solves_mincut_locally}
    Let $G$ be a graph with set of nodes $\{1,...,n\}$ and edge weight matrix $W \in \mathbb{M}_{n\times n}(\R)$. Suppose that $W$ has nonnegative entries on the diagonal. Then the $d$-dimensional Hopfield network $(W,0,\cl)$ with $n$ neurons operating in a serial mode decreases the $d$-cut value for the $d$-clustering of $G$ encoded by its matrix of neuron states at each step.
\end{corollary}
\begin{proof}
    The proof follows immediately from Fact \ref{fact:min_cut_quadratic_form} and Theorem \ref{theorem:energy_for_DHN_with_cl_activation}. \\
\end{proof}

\noindent
For the case when $B$ is nonzero, we first need to introduce some notation. Consider a weighted graph with nodes set $\{1,...,n+d\}$ and edge weight matrix in block form 
$$\begin{pmatrix} 
W & B \\ 
B^T & U
\end{pmatrix}$$
where $W \in \mathbb{M}_{n\times n}(\R),B \in \mathbb{M}_{n\times d}(\R),U\in \mathbb{M}_{d\times d}(\R)$ and $W, U$ are symmetric. We denote the graph obtained this way by $G\left(W,B,U\right)$. Next let $X \in \Clus_{n\times d}(\R)$ be a clustering matrix. Then the matrix written in the block form 
$$\begin{pmatrix} 
    X \\ 
    I_d \\
\end{pmatrix}$$
is called the canonical extension of $X$. Note that the canonical extension of $X$ is a clustering matrix in $\Clus_{(n + d)\times d}(\R)$. In particular, if $X \in \Clus_{n\times d}(\R)$, then its canonical extension encodes a $d$-clustering of $G(W,B,U)$.

\begin{theorem}\label{theorem:energy_clustering_optimality_equivalence}
    Let $(W, B, \cl)$ be a $d$-dimensional Hopfield network with $n$ neurons and let $U \in \mathbb{M}_{d\times d}(\R)$ be a symmetric matrix. Suppose that $W$ has nonnegative entries on the diagonal. Then the following assertions hold.
    \begin{enumerate}[label=\emph{\textbf{(\arabic*)}}, leftmargin=*]
        \item If $(W, B,\cl)$ operates in a serial mode, then the $d$-clusterings of $G(W,B,U)$ encoded by the canonical extensions of successive matrices of neuron states of $(W,B,\cl)$ have decreasing $d$-cut values.
        \item A state $X$ of $(W, B, \cl)$ is stable for every serial mode operation if and only if the $d$-cut value of the $d$-cut encoded by the canonical extension of $X$ in $G(W, B, U)$ cannot be decreased by moving any of the first $n$ nodes of $G(W, B, U)$ to another cluster.
        \item If
        $$U_{ij} = \left\{ \begin{array}{ll} -\kappa & \text{if } i \neq j \\ 0 & \text{otherwise} \end{array} \right.$$ 
        for sufficiently large $\kappa \in \R$, then there exists a state $X$ of $(W, B, \cl)$ which is stable for every serial mode operation, and the $d$-clustering of $G(W,B,U)$ encoded by the canonical extension of $X$ has globally minimal $d$-cut value.
    \end{enumerate}
\end{theorem}
\begin{proof}
    For brevity we denote $G(W,B,U)$ by $G$. Let $V$ be the energy function of $d$-dimensional Hopfield network $(W,B,\cl)$ described in Theorem \ref{theorem:energy_for_DHN_with_cl_activation}. Pick a matrix $X \in \Clus_{n\times d}(\R)$ and let $\{c_1,...,c_d\}$ be a clustering of $G$ encoded by the canonical extension of $X$. Then by Fact \ref{fact:min_cut_quadratic_form} we have
    $$C_{G}(c_1,...,c_d) = \mathrm{Vol}\left(G\right) - \mathrm{Tr}\left(
        \begin{pmatrix} 
        X^T& I_d^T 
        \end{pmatrix}\cdot
        \begin{pmatrix} 
            W & B \\ 
            B^T & U
        \end{pmatrix}\cdot
        \begin{pmatrix}
            X \\ 
            I_d  
        \end{pmatrix}
    \right) = $$
    $$= \mathrm{Vol}\left(G\right) - \mathrm{Tr}(U) + V(X)$$
    From this observation we can immediately infer \textbf{(1)} and \textbf{(2)}.\\
    In order to prove \textbf{(3)} consider $m,M\in \R$ such that
    $$m \leq - \mathrm{Tr}\left(X^TWX + 2X^TBY\right) \leq M$$
    for all $X \in \Clus_{n\times d}(\R)$ and $Y \in \Clus_{d\times d}(\R)$. Fix $\kappa \in \R$ satisfying
    $$m + \kappa > M$$
    and we assume that $U$ is of the special form as indicated in \textbf{(3)}. Let $c_1,...,c_d$ be a clustering of $G$ with minimal $d$-cut value. Let $X \in \Clus_{n\times d}(\R)$ and $Y \in \Clus_{d\times d}(\R)$ be matrices such that
    $$\begin{pmatrix} 
        X \\ 
        Y 
    \end{pmatrix} \in \Clus_{(n + d)\times d}(\R)$$
    encodes $\{c_1,...,c_d\}$. According to Fact \ref{fact:min_cut_quadratic_form} we have
    $$C_G(c_1,...,c_d) = \mathrm{Vol}\left(G\right) - \mathrm{Tr}\left(
        \begin{pmatrix} 
        X^T& Y^T 
        \end{pmatrix}\cdot
        \begin{pmatrix} 
            W & B \\ 
            B^T & U
        \end{pmatrix}\cdot
        \begin{pmatrix}
            X \\ 
            Y  
        \end{pmatrix}
    \right) = $$ 
    $$= \mathrm{Vol}\left(G\right) - \mathrm{Tr}\left(
        X^TWX + Y^TB^TX + X^TBY + Y^TVY
    \right)=$$
    $$= \mathrm{Vol}\left(G\right) - \mathrm{Tr}\left(X^TWX\right) - 2\mathrm{Tr}\left(X^TBY\right) - \mathrm{Tr}\left(Y^TUY\right) = $$
    $$= \mathrm{Vol}\left(G\right) - \mathrm{Tr}\left(X^TWX + 2X^TBY\right) + \kappa\cdot \sum_{k=1}^d\sum_{i \neq j}Y_{ik}Y_{jk}$$
    Now if $Y_{ik} = 1 = Y_{jk}$ for some $j\neq i$ and $k \in \{1,...,d\}$, then 
    $$C_G(c_1,...,c_d) \geq \mathrm{Vol}\left(G\right) + m + \kappa > \mathrm{Vol}\left(G\right) + M$$
    This contradicts the fact that $C_G$ achieves it's global minimum for $\{c_1,...,c_d\}$. Hence $Y$ is a permutation matrix and thus $Y^{-1} = Y^T$ is also a permutation matrix. Since 
    $$\begin{pmatrix}
        X \\ 
        Y  
    \end{pmatrix} \cdot Y^{-1} = \begin{pmatrix}
        XY^{-1} \\ 
        I_d  
    \end{pmatrix}$$
    and according to Remark \ref{remark:nonuniqueness_of_clustering_matrix}, we derive that the canonical extension of $XY^{-1}$ encodes $\{c_1,...,c_d\}$. By \textbf{(2)} we infer that $XY^{-1}$ is a stable state for $(W,B,\cl)$ for every serial mode of operation, and this completes the proof of \textbf{(3)}. 
\end{proof}

\subsection{Details on Louvain and Newman's methods}\label{appendix:details_on_louvain_and_newman}
\noindent
Intuitively, the modularity value associated with an edge measures the difference between the weight of an edge compared to a null-hypothesis obtained by a probabilistic approximation of the underlying graph \cite{newman:2004}. The null hypothesis is usually obtained using the configuration model. Given a graph, each edge is split in half, then each half edge (often referred to as stub) is rewired randomly with any other stub in the network. Thus, the baseline weight for a pair of nodes is given by the expected number of edges between the pair of nodes in a random rewiring of the stubs. A major advantage of approximating a graph in this way, is that the degrees of the nodes, and hence the degree distribution of the graph is preserved. \\
If an edge has a higher weight than the one expected based on the null-hypothesis, it is assumed to indicate an important connection (and a positive modularity value is assigned), whereas if the weight is smaller than the expected, it indicates smaller than expected connection between the two nodes (and a negative modularity value is assigned). More exactly:

\begin{definition}\label{definition:modularity_matrix}
    Let $n$ be a positive integer and let $G$ be a graph with set of nodes $\{1,...,n\}$ and edge weight matrix $W \in \mathbb{M}_{n\times n}(\R)$. Consider the matrix $Q \in \mathbb{M}_{n\times n}(\R)$ given by
    $$Q_{ij} =\frac{1}{\mathrm{Vol}(G)} \left( W_{ij} - \frac{k_i k_j}{\mathrm{Vol}(G)} \right)$$
    where $\mathrm{Vol}(G) = \sum_{i=1}^{n}\sum_{j=1}^{n} W_{ij}$ and $k_i = \sum_{j=1}^{n} W_{ij}$. Then $Q$ is \textit{a modularity matrix of $G$}.
\end{definition}
\noindent
This way, we can define the modularity of a clustering of $G$, as the sum of the intra-cluster modularities:

\begin{definition}\label{definition:modularity}
    Let $n$ be a positive integer and let $G$ be a graph with set of nodes $\{1,...,n\}$ and edge weight matrix $W \in \mathbb{M}_{n\times n}(\R)$. The \textit{modularity of a $d$-clustering $\{c_1,...,c_d\}$ of $G$} is defined as
    $$Q\left(c_1,...,c_d\right) = \sum_{k = 1}^d \sum_{i,j \in c_k}Q_{ij}$$
\end{definition}

\begin{remark}
    With the notation as in the definition above we have
    $$Q\left(c_1,...,c_d\right) = \sum_{k = 1}^d \sum_{i,j \in c_k}Q_{ij} = \sum_{i,j=1}^nQ_{ij} - \sum_{k\neq l}\sum_{i\in c_k}\sum_{j\in c_l}Q_{ij}$$
    and hence maximizing modularity of $G$ is the same as finding the minimal $d$-cut value of a graph which has the same set of vertices as $G$, but has edge weight matrix $Q$.
\end{remark}
\noindent
The Louvain method \cite{blondel:2008} is a popular graph clustering method, based around the heuristic idea of greedy local search, using the modularity of the clustering as an objective function to maximise. Strictly speaking, after reaching a local optimum, the Louvain method merges nodes in the same cluster together, but this part of the algorithm is irrelevant for our analysis, thus in the sequel we ignore it, and refer to Listing \ref{listing:louvain_method} as the Louvain method.
\begin{lstlisting}[numbers=left,caption={\label{listing:louvain_method}\textbf{Louvain method}},belowskip=5mm]
    initialise the clustering $\{c_i\}_{i=1}^{d}$(*\footnote{In practice this initialisation can take many forms. See for example Remark \ref{remark:louvain_initialisation}.}*)
    pick $u \in 1, \ldots, n$
        $m^* \leftarrow \text{argmax}_m \> Q(c_1 \setminus \{u\}, \ldots, c_{m-1} \setminus \{u\}, c_{m} \cup \{ u \}, c_{m + 1} \setminus \{ u \}, \ldots, c_d \setminus \{ u \})$(*\footnote{In original version of the algorithm only clusters with nonzero connection to $u$ are taken in argmax (neighboring clusters). This is equivalent to our version under the additional assumption that all nodes in the graph have nonnegative degrees. In \cite{blondel:2008} all weights are assumed to be nonnegative.}*)
        move node $u$ to cluster $c_{m^*}$
\end{lstlisting}

\begin{remark}\label{remark:louvain_initialisation}
    In \cite{blondel:2008} the initialisation takes the form $\{c_i\}_{i=1}^{n}$ with $c_i=\{ i \}$.
\end{remark}
\noindent
The steps $(2-4)$ resemble the update rule of a DHN with $\cl$ activation very closely. In fact, one can think of the Louvain method as a direct analogue of a DHN with classification function, as the following rigorous restatement of Proposition \ref{proposition:louvain_hopfield_net_equivalence} shows.

\begin{proposition}\label{proposition:louvain_hopfieldnet_equivalence}
    Let $G$ be a graph with nodes $\{1,...,n\}$, edge weight matrix $W \in \mathbb{M}_{n\times n}(\R)$ and modularity matrix $Q$. Suppose that $\tilde{Q}$ is the same matrix as $Q$ but with zeros on the diagonal. Consider an $n$-dimensional Hopfield network $(\tilde{Q}, 0, \cl)$ with $n$ neurons and pick $u \in \{1,...,n\}$. If $X$ is a clustering matrix, then running a single serial update of $(\tilde{Q},0,\cl)$ for neuron $u$ and matrix of neuron states $X$ produces the same result as updating clustering encoded by $X$ by Louvain method with respect to node $u$.
\end{proposition}

\begin{figure}[H]
    \centering
    \begin{tikzcd}[column sep=huge, row sep=normal]
        c \arrow[r, "\text{Louvain update}"]  & c'  \\
        X \arrow[r, "\text{DHN update}"] \arrow[u, dashed] & X' \arrow[u, dashed]
    \end{tikzcd}
    \caption{\label{diagram:louvain_equivalent_to_mhn}\textbf{Figure illustrating the claim of Proposition \ref{proposition:louvain_hopfieldnet_equivalence}}.}
\end{figure}

\begin{proof}
    Let $X \in \Clus_{n\times d}(\R)$ be a clustering matrix encoding clustering $\{c_1,...,c_d\}$ of $G$ and fix $u \in \{1,...,d\}$. Then
    $$Q\left(c_1 \setminus \{u\}, \ldots, c_{m-1} \setminus \{u\}, c_{m} \cup \{ u \}, c_{m + 1} \setminus \{ u \}, \ldots, c_d \setminus \{ u \}\right) = $$
    $$= 2\cdot \mathrm{Tr}(Q) + \sum_{k=1}^d\sum_{i,j\in c_k\setminus \{u\},i\neq j}Q_{ij} + 2\cdot \sum_{j\in c_m\setminus \{u\}}Q_{ju}$$
    $$= 2\cdot \mathrm{Tr}(Q) + \sum_{k=1}^d\sum_{i,j\in c_k\setminus \{u\},i\neq j}Q_{ij} + 2\cdot \sum_{j\in c_m}\tilde{Q}_{ju}$$
    and thus
    $$\argmax_{m}Q\left(c_1 \setminus \{u\}, \ldots, c_{m-1} \setminus \{u\}, c_{m} \cup \{ u \}, c_{m + 1} \setminus \{ u \}, \ldots, c_d \setminus \{ u \}\right) = $$
    $$= \argmax_{m}\sum_{j\in c_m}\tilde{Q}_{uj} = \argmax_m \sum_{j=1}^n\tilde{Q}_{uj}\cdot X_{jm}$$
    Therefore, updating $u$ according to the serial update rule for $(\tilde{Q},0,\cl)$ and updating it by means of Louvain update rule yields the same result.\\
\end{proof}

\begin{corollary}\label{corollary:louvain_method_as_mhn_with_cl}
    Using the same setting as in Proposition \ref{proposition:louvain_hopfieldnet_equivalence}, DHN $(\tilde{Q}, 0, \cl)$ running in serial mode with $I_n$ as the initial matrix of neuron states runs one whole iteration of of the Louvain method with initialization as in original implementation.
\end{corollary}
\begin{proof}
    Since $I_n \in \Clus_{n \times n}(\R)$ is a clustering matrix corresponding to the state when every vertex is in a different cluster, thus initial neuron states matrix of $(\tilde{Q}, 0, \cl)$ encodes the initial state of the Louvain method by Remark \ref{remark:louvain_initialisation}. According to Proposition \ref{proposition:louvain_hopfieldnet_equivalence} this property is preserved after every iteration.
\end{proof}
\noindent
Another popular method frequently applied for finding clusters with high modularity scores is associated to Newman \cite{newman:2006, newman:2006:2}. Let $G$ be a graph as above and let $Q$ be its modularity matrix. Newman's method relies on the fact, that finding a $2$-clustering of $G$ with maximal modularity can be written in terms of maximising quadratic form 
$$s^T Q s$$ 
subject to constraint $s \in \{-1,1\}^n$. Since $Q$ is real and symmetric $n\times n$ matrix, it has a complete set of orthonormal eigenvectors, say $\bd{u}_1, ...,\bd{u}_n$, with corresponding eigenvalues $\lambda_1 \geq \cdots \geq \lambda_n$. Then, for any vector $s \in \R^n$, we can write 
$$s = a_1 \bd{u}_1 + \cdots + a_n \bd{u}_n$$
where $a_i \in \R$. Then we have
$$s^T Q s = \left( \sum_{i=1}^{n} a_i \bd{u}_i \right)^T\cdot Q \cdot \left( \sum_{i=1}^{n} a_i \bd{u}_i \right) = \sum_{i=1}^{n} a_i^2 \lambda_i$$
If $\lVert s\rVert_2 = 1$ and since $\bd{u}_i$ are orthogonal, we derive that $\sum_{i=1}^{n} a_i^2 = 1$. It follows that 
$$\bd{u}_1 = \argmax_{\lVert s \rVert _2 = 1}s^TQs$$
Thus one can approach solving
$$\max_{s\in \{-1,1\}^n}s^TQs$$
by first finding the eigenvector of $Q$ with the largest eigenvalue, and then finding the element of $\{-1, 1\}^n$ which, in geometric terms, aligns with it the most. That is:
$$\argmax_{s \in \{-1, 1\}^n} s^T \bd{u}_1$$
This has the closed form solution given by $\sgn \left( \bd{u}_1 \right)$ where $\sgn$ is applied component-wise (see \cite{newman:2006}). \\
Under optimal circumstances the method is guaranteed to converge quickly. A major drawback of the algorithm, is that while spectral methods provide both a conceptually simple and easy-to-solve framework for finding a graph cut with high modularity, currently clustering into more than two parts is achieved by iteratively splitting clusters, as solving for more than two clusters is computationally hard \cite{newman:2006:2}.
\noindent
In practice, calculating the leading eigenvector is achieved by applying the power-method \cite[pp. 194-198]{allaire:2008} on $Q$. By construction, $Q$ can be written as:
$$Q = W - \frac{k k^T}{\mathrm{Vol}(G)}$$
Where $k = (k_1, ..., k_n)^T$ is the degree vector of $G$ (see Definition \ref{definition:modularity_matrix}). This way, assuming that $W$ is sparse, which is very often the case in real-world applications, the time complexity of multiplying by $Q$ reduces significantly. The listing below shows the pseudo-code implementation of the method.

\begin{lstlisting}[numbers=left,caption={\label{listing:newman_method}\textbf{Newman method}},belowskip=5mm]
    input $Q$
    initialise $\bd{v}$ randomly
    until $\bd{v}/ \lVert \bd{v} \rVert_2$ converges
        $\bd{v} = Q\bd{v}$
    return $\sgn(\bd{v})$
\end{lstlisting}
\begin{remark}
    In implementations, the 4th line of Listing \ref{listing:newman_method} is replaced by 
    $$\bd{v} = Q\bd{v}/||Q\bd{v}||_2$$
    due to practical reasons. From the purely theoretical point of view, which is our main concern, this modification does not influence the output. Therefore, we decide to get rid of it.
\end{remark}
\noindent
Now, consider the 1-dimensional Hopfield network $(Q, \bd{0}, \text{id})$ running in parallel mode. According to Remark \ref{remark:activation_application_to_rows}, the network in parallel mode is equivalent to lines 3-4 in Listing \ref{listing:newman_method}. After convergence\footnote{Convergence is understood as the 'direction' of $X$ being stuck in a cycle. More precisely, we halt the loop when $\lVert \hat{X}(t) - \hat{X}(t-k) \rVert < \varepsilon$, where $t$ is the current timestamp, $k \in \{1, ..., N\}$ for some fixed $N \in \NN$, $\varepsilon > 0$, $\hat{X} = X/\lVert X \rVert$ and $\lVert . \rVert$ is the Frobenius norm . The exact values the parameters $\varepsilon$ and $N$ can be finetuned based on the nature of the application.}, the resulting vector, say $\bd{v}$, can't be used for clustering yet, since it's elements are continuous real values, therefore we apply $\sgn$, to obtain the cut most resembling the output of $(Q, \bd{0}, \text{id})$.
Expressing the Newman method this way allows us to extend the method to the case of multi-label clustering, simply by considering a DHN of higher dimensions. When collapsing the final matrix of neuron states, though, this time our target space is the space of clustering matrices of appropriate dimensions. We provide following proposition about $\cl$, that is a natural generalisation of the claim that $\max_{s\in \{-1, 1\}^n} s^T \bd{u}_1$ has closed form solution $\sgn(\bd{u}_1)$.

\begin{proposition}\label{proposition:cl_projection_and_maximal_align}
    Let $M \in \mathbb{M}_{n \times d}(\R)$. Suppose that $\cl(M)$ denotes the matrix obtained by application of $\cl$ to rows of $M$. Then 
    $$\cl(M) = \argmax_{S \in \Clus_{n\times d}(\R)} \text{Tr} \left( S^T M \right)$$
\end{proposition}
\begin{proof}
    Clearly $\cl(M) \in \Clus_{n\times d}(\R)$ by Remark \ref{remark:clustering_matricies_is_range_of_classification}. Now, consider any $S \in \Clus_{n\times d}(\R)$. Let $\sigma_{S}: \{1, \ldots, n\} \to \{1, \ldots d\}$ be a map given by formula
    $${S}_{ij} = \left\{ \begin{array}{ll} 1 & \text{if } \sigma_{S}(i) = j \\ 0 & \text{otherwise} \end{array} \right.$$
    and note that
    $$\mathrm{Tr}\left(S^TM\right) = \sum_{i=1}^{m} \sum_{j=1}^{n} S^T_{ij} M_{ji} = \sum_{i=1}^{m} \sum_{j=1}^{n} S_{ji} M_{ji} = \sum_{j=1}^{n}  M_{j\sigma_{S}(j)}$$
    By definition of $\cl$ we have
    $$\sum_{j=1}^{n}  M_{j\sigma_{S}(j)} \leq \sum_{j=1}^{n}  M_{j\sigma_{\cl(M)}(j)}$$
    and hence
    $$\mathrm{Tr}\left(S^TM\right) \leq \mathrm{Tr}\left(\cl(M)^TM\right)$$
    for every $S \in \Clus_{n\times d}(\R)$. This completes the proof.
\end{proof}

\begin{remark}\label{remark:cl_is_projection_in_hilbert_space_of_n_d_matrices}
The proposition above can be also reformulated in the following way. Note that $\mathbb{M}_{n\times d}(\R)$ is a Hilbert space with scalar product
$$\langle M_1, M_2 \rangle = \mathrm{Tr}\left(M_1^TM_2\right)$$
which induces Frobenius norm 
$$\lVert M\rVert_F = \sqrt{\mathrm{Tr}\left(M^TM\right)}= \sqrt{\sum_{i,j=1}^nM_{ij}^2}$$
for $M \in \mathbb{M}_{n\times d}(\R)$. Then Proposition \ref{proposition:cl_projection_and_maximal_align} is equivalent to the fact that the map $\cl:\mathbb{M}_{n\times d}(\R)\twoheadrightarrow \Clus_{n\times d}(\R)$ satisfies
$$\lVert \cl(M) - M \rVert_F \leq \lVert S - M \rVert_F$$
for all $S \in \Clus_{n\times d}(\R)$.
\end{remark}
\noindent
\noindent
Using this information it is straightforward to write Newman's method as a DHN. When extending to multiple dimensions, though, one has to take care that the extra dimensions are somehow 'coupled', otherwise they all converge to the leading eigenvector. See the main text for a possible solution.
\end{document}